\documentclass[letterpaper]{article} 
\usepackage{aaai2026}  
\usepackage{times}  
\usepackage{helvet}  
\usepackage{courier}  
\usepackage[hyphens]{url}  
\usepackage{graphicx} 
\usepackage{natbib}  
\usepackage{caption} 
\usepackage{algorithm}
\usepackage{algorithmic}

\usepackage{newfloat}
\usepackage{listings}
\DeclareCaptionStyle{ruled}{labelfont=normalfont,labelsep=colon,strut=off} 
\floatstyle{ruled}
\newfloat{listing}{tb}{lst}{}
\floatname{listing}{Listing}
\title{Breadth-First Search vs. Restarting Random Walks for Escaping Uninformed Heuristic Regions}
\author{
    Daniel Platnick\textsuperscript{\rm $\mathsection$, $\dagger$,}\footnote{Work done while at Toronto Metropolitan University},
    Dawson Tomasz\textsuperscript{$\dagger$, $\ddagger$},
    Eamon Earl\textsuperscript{$\dagger$, $\ddagger$},
    Sourena Khanzadeh\textsuperscript{$\dagger$},
    Richard Valenzano\textsuperscript{$\dagger$, $\ddagger$}
}
\affiliations{
    \textsuperscript{\rm $\dagger$}Toronto Metropolitan University,
    \textsuperscript{\rm $\ddagger$}Vector Institute,
    \textsuperscript{\rm $\mathsection$} Flybits Labs\\
    daniel.platnick@flybits.com, dawson.tomasz@proton.me, $\lbrace$eamon.earl, sourena.khanzadeh, rick.valenzano$\rbrace$@torontomu.ca
}

\usepackage{bibentry}

\usepackage{amsmath}
\usepackage{amsfonts}
\usepackage{amssymb}
\usepackage{amsthm}
\usepackage{subcaption}
\usepackage{multirow}
\usepackage{booktabs}
\usepackage[table,xcdraw]{xcolor}
\usepackage{arydshln}

\newtheorem{theorem}{Theorem}[section]
\newtheorem*{theorem*}{Theorem}

\newtheorem{corollary}[theorem]{Corollary}
\newtheorem{lemma}[theorem]{Lemma}
\newtheorem{observation}{Observation}[section]

\theoremstyle{definition}
\newtheorem{definition}{Definition}[section]

\newcommand{\expect}{\mathbb{E}}
\newcommand{\prob}{\mathbb{P}}

\newcommand{\tuple}[1]{\langle{#1}\rangle}

\newcommand{\OPEN}{\mathrm{OPEN}}
\newcommand{\CLOSED}{\mathrm{CLOSED}}

\newcommand{\parent}{\mathrm{parent}}
\newcommand{\NONE}{\mathrm{NONE}}

\newcommand{\task}{\mathcal{T}}
\newcommand{\states}{\mathcal{S}}
\newcommand{\transitions}{\Delta}
\newcommand{\goaltest}{\Gamma}
\newcommand{\true}{\mathrm{True}}
\newcommand{\false}{\mathrm{False}}
\newcommand{\initstate}{s_{\mathcal{I}}}

\newcommand{\goaldepth}{d^*}

\newcommand{\goaldepthstates}{S_{\goaldepth}}
\newcommand{\shallowstates}{S_{<\goaldepth}}
\newcommand{\numgoals}{g}
\newcommand{\bfactor}{b}
\newcommand{\bfsrv}{B}
\newcommand{\walklength}{\ell}
\newcommand{\rrwrv}{R^C_\walklength}
\newcommand{\successprob}{p_g}
\newcommand{\goaldepthprob}{p_{d^*}}

\newcommand{\crrw}{RRW$^C_\walklength$}
\newcommand{\crrwspace}{\crrw~}
\newcommand{\lubyrrw}{RRW$^\mathcal{L}$}
\newcommand{\lubyrrwspace}{\lubyrrw~}

\begin{document}

\maketitle

\begin{abstract}
Greedy search methods like Greedy Best-First Search (GBFS) and Enforced Hill-Climbing (EHC) often struggle when faced with Uninformed Heuristic Regions (UHRs) like heuristic local minima or plateaus. 
In this work, we theoretically and empirically compare two popular methods for escaping UHRs in breadth-first search (BrFS) and restarting random walks (RRWs). 
We first derive the expected runtime of escaping a UHR using BrFS and RRWs, based on properties of the UHR and the random walk procedure, and then use these results to identify when RRWs will be faster in expectation than BrFS. 
We then evaluate these methods for escaping UHRs by comparing standard EHC, which uses BrFS to escape UHRs, to variants of EHC called EHC-RRW, which use RRWs for that purpose. EHC-RRW is shown to have strong expected runtime guarantees in cases where EHC has previously been shown to be effective. We also run experiments with these approaches on PDDL planning benchmarks to better understand their relative effectiveness for escaping UHRs.
\end{abstract}

\section{Introduction}
When given a reasonably accurate heuristic function, greedy algorithms like \emph{Greedy Best First Search (GBFS)} \cite{doran:gbfs} and \emph{Enhanced Hill-Climbing (EHC)} \cite{hoffmannff_2001} can be effective at solving planning problems.
However, when using a flawed heuristic function, these methods can become stalled due to \textit{Uninformative Heuristic Regions (UHRs)} in which the heuristic provides no or flawed guidance.
Notably, EHC is explicitly designed to perform a Breadth-First Search (BrFS) to find a way out of UHRs.
When GBFS gets stuck in a plateau --- which is a UHR in which all states have the same heuristic value --- it too will degenerate into BrFS when using low-g tiebreaking.

\emph{Restarting Random Walks (RRW)} have also effectively been used to escape UHRs.
In GBFS, RRWs have been initiated when the search stops seeing improvement in the heuristic values of the states encountered \cite{xie:local_rws}.
Alternatively, Arvand is an EHC-like local search that uses RRWs to progress through the state-space \cite{nakhost:arvand,nakhost:arvand_2nd_gen}.

In practice, BrFS-based approaches are more effective on some problems and RRW-based are more effective on others.
But more work is needed to improve our understanding of why these differences occur or when we would expect either method to be the better option.
This work aims to help address this gap, with the goal of making it clearer when RRWs should be deployed, or when they should be used alongside BrFS in an \emph{algorithm portfolio}.
To that end, we make the following contributions:

\begin{enumerate}
    \item We identify expected runtimes for BrFS and \emph{constant-depth RRWs} for escaping a UHR with uniformly distributed exits. These results are given in terms of the size of the UHR and the \emph{success probability}, which is the probability that a single random walk escapes the UHR.
    \item We show that RRWs will be faster in expectation than BrFS if the success probability is larger than the ratio of the number of states at shallower depths than the first escape state and the number of states at the shallowest escape depth, called the \emph{goal depth}. We give a lower bound for directed trees using unbiased random walks on the number of escapes at the goal depth that guarantees the success probability is high enough for RRWs to be faster.
    \item We compare variants of EHC that use constant-depth RRWs or the popular Luby restart policy \cite{luby:restarts} instead of BrFS to escape UHRs. We show that these methods have strong expected runtime guarantees in cases where EHC is known to be complete or have a polynomial runtime for STRIPS planning.
    \item We empirically compare EHC with the RRW variants on PDDL problems to show how our theoretical results apply in practice, and how different state-space topological features relatively influence EHC and the RRW variants.
\end{enumerate}

\section{Preliminaries}
In this section, we introduce our terminology and notation, and describe the algorithms and methods analyzed below.

\subsection{Search Tasks and State-Space Topologies}

A \emph{state-space search task} $\task$ is defined by the tuple $\task = \tuple{\states, \initstate, \transitions, \goaltest}$, where $\states$ is a finite set of \emph{states}, $\initstate$ is the initial state, $\transitions : \states \rightarrow 2^{\states}$ is the \emph{state transition function}, and $\goaltest: \states \rightarrow \lbrace \true, \false \rbrace$ is the \emph{goal test function}.
If $s'$ is in $\transitions(s)$, we refer to $s'$ as a \emph{successor} of $s$.
When $\transitions$ is used to find the successors of $s$, we say those successors are \emph{generated}.

A path $P = \tuple{s_0,\ldots, s_k}$ is a sequence of states where $s_i \in \transitions(s_{i -1})$ for every $0 < i \leq k$.
The objective of a given task $\task$ is to find a \emph{solution path}, where $s_0 = \initstate$ and $s_k$ is a \emph{goal state} (\textit{i.e.} $\goaltest(s_k) = \true$).
As our focus is on \emph{satisficing} search --- in which we want to find any solution regardless of cost --- we ignore transition costs in this work.
In addition, we let $last(P)$ denote the last state on $P$ (\textit{i.e.} $last(P) = s_k$).
If $P' = \tuple{s_0', \ldots, s_j'}$, then $P + P'$ is the concatenation of the paths $P$ and $P'$.
For brevity, we abuse notation and let $P + P' = \tuple{s_0, ..., s_{k-1}, s_0', \ldots s_j'}$ if the last state on $P$ is the same as the first state on $P'$ (\textit{i.e.} $s_k = s_0'$).

We now define several important state-space properties.
For any $s \in \states$, the \emph{depth} of $s$ is the number of transitions in the shortest path from $\initstate$ to $s$.
For example, $\initstate$ has a depth of $0$, any state in $\transitions(\initstate)$ has a depth of $1$, etc.
A state $s$ is called a \emph{dead end} if no goal state is reachable from $s$.
The \emph{goal depth} $\goaldepth$ of a task $\task$ is defined as the minimum depth of any goal state.
We also denote the set of unique states with a depth strictly less than $\goaldepth$ as $\shallowstates \subseteq \states$, and the set of unique states with a depth exactly equal to $\goaldepth$ as $\goaldepthstates \subseteq \states$.

A \emph{state-space topology} is a pair $\tuple{\task,h}$, where $\task$ is a search task and $h: \states \rightarrow \mathbb{Z}^{\geq 0} \cup \lbrace \infty \rbrace$ is a \emph{heuristic function}.
Below, we assume that $h$ never incorrectly identifies a state as a dead end, meaning if $h(s) = \infty$, then no goal state is reachable from $s$.
While $h$ will ideally provide useful search guidance, \emph{Uninformative Heuristic Regions} (UHRs) do occur.
A UHR around any state $s \in \states$ is the set of states reachable from $s$ along any path $P$ such that for any $s' \in P$, $h(s') \geq h(s)$.
That is, no ``heuristic progress" occurs while in a UHR.
A state $s_e$ in the UHR is called an \emph{exit} if $s_e$ has a successor $s''$ for which $h(s'') < h(s)$.
We refer to any such ``improving" successor $s''$ of $s_e$ as an \emph{escape state}.
The length of the shortest path from $s$ to any exit is also referred to as the \emph{exit distance} of the UHR.

\subsection{Search Algorithms and Methods}

We now briefly describe the search methods of focus below.

\begin{algorithm}[t]
\small
\begin{algorithmic}[1]
    \STATE \textbf{Input:} task $\tuple{\states, \initstate, \transitions, \goaltest}$, heuristic $h$
    \STATE $P \gets \tuple{\initstate}$, $s \gets \initstate$
    \WHILE{$\true$}
        \STATE $\task' \gets \tuple{\states, s, \transitions, \goaltest^h_s}$
        \STATE $P' \gets$ \texttt{brFS}($\task'$) \label{ehc:brfs_call}
        \IF{$P' = \tuple{}$}
            \RETURN $\tuple{}$~~~~~~\% solution not found
        \ENDIF
        \STATE $P \gets P + P'$, $s \gets last(P')$
        \IF{$\goaltest(last(P)) = \true$}
            \RETURN $P$
        \ENDIF
    \ENDWHILE
    \RETURN $\tuple{}$~~~~~~\% No solution found
\end{algorithmic}
\caption{Enforced Hill-Climbing}
\label{alg:ehc}
\end{algorithm}

\subsubsection{Breadth-First Search (BrFS).} We assume the reader's familiarity with BrFS, though pseudocode is given in Appendix \ref{sec:appendix_algs}.
Notably, we define BrFS as a \emph{best-first search} where a state's priority is giveb by its depth.
BrFS is also defined to perform a goal test on a state $s$ when it is generated, not when $s$ is selected for expansion.
BrFS will still find the shortest path when modified in this way.

\subsubsection{Enforced Hill-Climbing (EHC).} This local search method was originally used in the FF planner \cite{hoffmannff_2001}.
Given a state-space topology $\tuple{\task,h}$, EHC performs a sequence of BrFSs, each aiming to escape the current UHR (see Algorithm \ref{alg:ehc}).
Importantly, instead of using the goal test $\goaltest$ given for the overall task, each BrFS uses the following goal test function nstead:
\begin{align} \label{eq:uhr_goal_test}
    \goaltest^h_s(s') = 
    \begin{cases}
    \true, & \text{ if } \goaltest(s') \text{ or } h(s') < h(s) \\
    \false, & \text{ otherwise}
    \end{cases}
\end{align}
$\goaltest^h_s$ succeeds when either a goal state according to $\goaltest$ is found, or an escape state is found for the current UHR (\textit{i.e.} ``heuristic progress" is made).
Thus, EHC searches for a sequence of escape states until a goal state is reached.

\begin{algorithm}[t]
\small
\begin{algorithmic}[1]
    \STATE \textbf{Input:} task $\tuple{\states, \initstate, \transitions, \goaltest}$
    \IF{$\goaltest(\initstate) = \true$} \label{rrw:goal_test_start}
        \RETURN $\tuple{\initstate}$ ~~~~~~\% Single state path is solution
    \ENDIF
    \WHILE{$\true$}
        \STATE $P \gets \tuple{\initstate}$, $s \gets \initstate$, $\walklength \gets$ \texttt{getDepth()}, $d \gets 0$ \label{rrw:get_depth}
        \WHILE{$d < \walklength$ and $|\transitions(s)| > 0$}
            \STATE $s' \gets$ state sampled from $\transitions(s)$
            \STATE $P \gets P + \tuple{s'}$
            \IF{$\goaltest(s') = \true$} \label{rrw:goal_test}
                \STATE \textbf{return} $P$
            \ENDIF
            \STATE $s \gets s'$, $d \gets d + 1$ 
        \ENDWHILE
    \ENDWHILE
\end{algorithmic}
\caption{Restarting Random-Walks}
\label{alg:rrw}
\end{algorithm}

\subsubsection{Restarting Random Walks (RRWs).} A \emph{random walk} is a single path through a state-space that is generated stochastically (lines 5 to 12 of Algorithm \ref{alg:rrw}).
At every step of the walk, a successor of the last state is sampled and added to the current path.
A random walk terminates when either a goal state is found, a state without any successors is encountered, or some maximum depth is reached.
A random walk is said to be \emph{unbiased} if the states are sampled uniformly over the set of possible successors.
We also let $0 \leq \successprob \leq 1$ denote the \emph{success probability} that the random walk will reach a goal.
Note that $\successprob$ may depend on the structure of the state-space (\textit{i.e.} the distribution of goals) or the way successors are sampled for the random walk (\textit{i.e.} unbiased or biased).

A \emph{restarting random walk (RRW)} performs a sequence of random walks, each starting from $\initstate$ (see Algorithm \ref{alg:rrw}).
An RRW terminates when any random walk reaches a goal state.
The maximum length of each random walk is determined by a call to 
\texttt{getDepth()} (line \ref{rrw:get_depth}).
When using \emph{constant-depth RRWs}  --- denoted as \crrwspace --- \texttt{getDepth()} always returns the same integer constant $\walklength > 0$.

\subsubsection{The Luby Restart Policy.} This strategy, which alters the depth limit from walk to walk, was originally defined for a general class of stochastic algorithms \cite{luby:restarts}.
In the context of random walks, \citeauthor{luby:restarts} showed that for any random walk procedure, there exists a constant $\walklength^* > 0$, such that always restarting after $\walklength^*$ steps has the minimum expected runtime over all possible \emph{non-adaptive} restart policies.
This means random walks are independent, and the restart policy does not change based on what is encountered during these walks.

Unfortunately, determining $\walklength^*$ for a given problem requires full knowledge of the runtime distribution of a single infinite length random walk on that problem.
As this is not known prior to search, \citeauthor{luby:restarts} introduced a general restart policy for unknown runtime distributions.
We omit the full details of the policy, but note that it is based on a sequence whose first 15 values are $\left\langle 1,1,2,1,1,2,4,1,1,2,1,1,2,4,8,\dots \right\rangle$.
The resulting Luby restart policy is used for RRWs by having the length of the $i$-th random walk be the $i$-th value in the sequence.
We refer to this algorithm as \lubyrrw.
Importantly, \citeauthor{luby:restarts} showed that if $T^*$ is the expected runtime when always restarting after $\walklength^*$ steps, then the expected runtime when using the Luby sequence is $O(T^*\log T^*)$. 

Intuitively, this approach performs longer and longer random walks to allow the search to reach deep goals if needed, while continually performing short walks to ensure shallow goals are not missed.
In practice it is common to multiply all values in the sequence by some integer constant $m \geq 1$ to reach greater depths faster. This approach has the same runtime guarantees as when using the original sequence.

\section{Expected Runtime Analysis} \label{runtime}

In this section, we characterize the expected runtime of BrFS and \crrwspace in terms of task size and random walk properties.
We then find bounds the success probability that guarantees that \crrwspace will be faster in expectation than BrFS.
This result is further refined in the case of unbiased random walks on a directed tree.
We conclude the section with a discussion of the implications and limitations of this analysis.

We note that our results are given in terms of solving a search task, not just escaping a UHR.
However, they also cover this case because the problem of escaping a UHR starting at state $s$ can be modeled as a search task whose objective is to find a state with a lower heuristic value than $h(s)$ (\textit{i.e.} by using the goal test function in Equation \ref{eq:uhr_goal_test}).

Below, use $\bfsrv(\task)$ and $\rrwrv(\task)$ for the random variables (RVs) of the runtime of BrFS and \crrw, respectively.
We measure runtime in terms of the number of goal tests performed or equivalently, the number of states generated.

\subsection{Expected Runtimes for BrFS and \crrw}

We begin with the following result for BrFS when the goal states are uniformly distributed at the goal depth:

\begin{theorem} \label{thm:brfs_exp}
If $\task$ has $\numgoals \geq 1$ goal states uniformly distributed among the $|\goaldepthstates|$ states at the goal depth, then
\begin{align*}
    \expect[\bfsrv(\task)] = |\shallowstates| + (|\goaldepthstates| + 1)/(\numgoals + 1)
\end{align*}
\end{theorem}
\begin{proof}
Let $X$ be the number of goal tests that BrFS performs on states at depth $\goaldepth$. Since BrFS examines all states shallower than $\goaldepth$ and none deeper than $\goaldepth$, it follows that $\bfsrv(\task) = |\shallowstates| + X$.
Thus $\expect[\bfsrv(\task)] = |\shallowstates| + \expect[X]$ since $|\shallowstates|$ is a constant.
We also note that  $1 \leq X \leq |\goaldepthstates|$ even if the goals are not uniformly distributed, since at least one state at the goal depth will be tested, and at worst all states at depth $\goaldepth$ will be tested. This means that
\begin{align}
    |\shallowstates| + 1 \leq \bfsrv(\task) \leq |\shallowstates| + |\goaldepthstates| \label{eqn:brfs_upper_and_lower}
\end{align}

If the goal states are uniformly distributed at the goal depth, $\expect[X]$ is equivalent to the expected number of selections needed when randomly picking states from the goal depth without replacement, until one of the $\numgoals$ goal states is picked. 
Where $s_i$ is any one of the $(|\goaldepthstates|- \numgoals)$ non-goal states in $\goaldepthstates$, $Z_i$ be an indicator RV for the event that $s_i$ is picked before any of the $\numgoals$ goals.
Therefore, $\expect[X] = \expect[Z_1 + ... + Z_{|\goaldepthstates| - \numgoals}] + 1$ since $X$ is the number of non-goal states tested plus one for the selected goal state.
Notice that $\prob[Z_i] = 1 / (\numgoals + 1)$ since there are $(\numgoals + 1)!$ ways of ordering the $(g + 1)$ states in the set containing $s_i$ and the $\numgoals$ goals, and $\numgoals!$ of these orderings start with $s_i$.
It thus holds that
\begin{align}
    \expect[X] &   = 1 + \expect[\sum_{i = 1}^{|\goaldepthstates| - \numgoals} Z_i] = 1 + \sum_{i = 1}^{|\goaldepthstates| - \numgoals} \expect[Z_i]\\
    & = 1 + (|\goaldepthstates| - \numgoals) / (\numgoals + 1) = (|\goaldepthstates| + 1) / (\numgoals + 1)\label{brfs_ex_proof:ind}
\end{align}
Line \ref{brfs_ex_proof:ind} holds since the $Z_i$ are indicator variables and so $\expect[Z_i]=\prob[Z_i]$, and since we are summing over $(|\goaldepthstates| - \numgoals)$ of RVs that all have the same expectation.
Adding this to $|\shallowstates|$ yields the desired result.
\end{proof}

The theorem shows that the expected runtime decreases as the density of goals at depth $\goaldepth$ (\textit{i.e.} $\numgoals/|\goaldepthstates|$) increases, since fewer states in $\goaldepthstates$ will likely need to be tested.
However, regardless of this density, BrFS must still exhaustively examine all states shallower than $\goaldepth$ (see equation \ref{eqn:brfs_upper_and_lower}).

We now turn to \crrw.
If the success probability of a single random walk is $\successprob = 0$ (\textit{i.e.} $\walklength < \goaldepth$), then $\expect[\rrwrv(\task)] = \infty$.
Otherwise, we can say the following:

\begin{theorem} \label{thm:rrw_exp}
If the success probability of a random walk to depth $\walklength$ on search task $\task$ is $\successprob > 0$, then
\begin{align*}
    \expect[\rrwrv(\task)] \leq \walklength/\successprob + 1
\end{align*}
\end{theorem}
\begin{proof}
Let $Y$ be the RV for the \emph{number of random walks} it takes to find a goal when using \crrw, $L$ be the RV for the length of a random walk given it a goal, and $\bar{L}$ be the RV for the length of a random walk given it does not reach a goal.
\crrwspace will perform $(Y - 1)$ walks of length $\bar{L}$ and one random walk of length $L$, and so
\begin{align}
    \expect[\rrwrv(\task)]  &  = \expect[(Y - 1)\bar{L} + L + 1]\label{rrw_exp:first}\\
    & = \expect[Y - 1]\expect[\bar{L}] + \expect[L] + 1 \label{rrw_exp:ind}\\
    & = (1/\successprob -  1)\expect[\bar{L}] + \expect[L] + 1 \label{rrw_exp:geometric}\\
    & \leq \walklength/\successprob + 1
\end{align}
\noindent
The additional $1$ in line \ref{rrw_exp:first} comes from the single goal test of $\initstate$ on line \ref{rrw:goal_test_start} of Algorithm \ref{alg:rrw}.
The random walks themselves are independent and identically distributed (IID), and so the length of each walk that does not reach a goal is independent of the number performed. 
As such, $Y$ and $\bar{L}$ are independent and Line \ref{rrw_exp:ind} holds.
The IID property also means that $Y$ follows the geometric distribution, and so $\expect[Y] = 1/\successprob$ (line \ref{rrw_exp:geometric}).
The final line holds because all walks have a length of at most $\walklength$ and so $\expect[L]$ and $\expect[\bar{L}]$ are both at most $\walklength$.
\end{proof}

\subsection{Comparative Analysis}
\begin{figure*}[t!]
 	\centering
 	\begin{subfigure}[t]{0.3\textwidth}
 		\centering
 		\includegraphics[width=\linewidth]{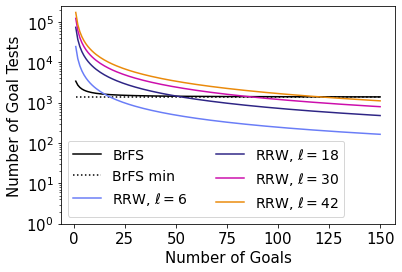} 
        \caption{Expected runtime for different numbers of goals where $\goaldepth=6$.}\label{fig:goal_behaviour}
 	\end{subfigure}
 	\hfill
 	\begin{subfigure}[t]{0.3\textwidth}
 		\centering
 		\includegraphics[width=\linewidth]{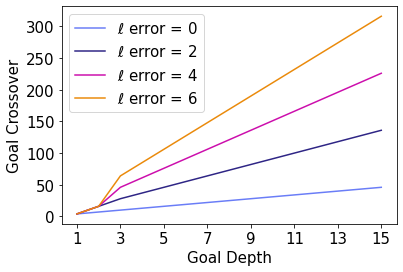} 
        \caption{The goal crossover point for different goal depths.}\label{fig:cutoff_behaviour}
 	\end{subfigure}
 	\hfill
 	\begin{subfigure}[t]{0.3\textwidth}
 		\centering
 		\includegraphics[width=\linewidth]{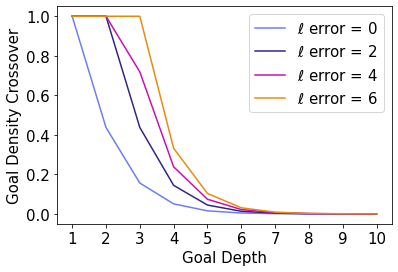} 
        \caption{The goal density crossover point for different goal depths.}\label{fig:density_behaviour}
 	\end{subfigure}
 	\caption{BrFS and \crrwspace with unbiased random walks on a directed tree with a branching factor of 4.}
\end{figure*}
To better understand Theorems \ref{thm:brfs_exp} and \ref{thm:rrw_exp}, consider a search task $\task_\bfactor$ on a directed tree with constant branching factor $\bfactor \geq 2$, and $\numgoals \geq 1$ goals uniformly distributed at depth $\goaldepth$.
There are $\bfactor^d$ states at every depth $d \geq 0$ of such a tree, meaning that $|\goaldepthstates| = \bfactor^{\goaldepth}$ and $|\shallowstates| = \bfactor^0 + \bfactor^1 + ... + \bfactor^{\goaldepth - 1} = (\bfactor^{\goaldepth} - 1)/(\bfactor - 1)$.  
If we are using unbiased random walks length $\ell \geq \goaldepth$ on such a tree, then $\successprob = \numgoals/\bfactor^{\goaldepth}$. Therefore, Theorems \ref{thm:brfs_exp}, and \ref{thm:rrw_exp} imply the following
\begin{align}
    \expect[\bfsrv(\task_\bfactor)] &= (\bfactor^{\goaldepth} - 1)/(\bfactor - 1) + (\bfactor^{\goaldepth} + 1) / (\numgoals + 1) \label{eqn:brfs_trees_exp} \\
    \expect[\rrwrv(\task)] & \leq \walklength \bfactor^{\goaldepth}/\numgoals + 1 \label{eqn:crrw_trees_exp}
\end{align}

Figure \ref{fig:goal_behaviour} shows the expected runtime of BrFS and \crrwspace with unbiased walks on such a tree --- as calculated using equations \ref{eqn:brfs_trees_exp} and \ref{eqn:crrw_trees_exp} --- with a constant branching factor $\bfactor=4$, $\goaldepth=6$, and different numbers of goals uniformly distributed among the 4096 states at depth $6$. 
BrFS is significantly faster when there are very few goals, but \crrwspace quickly catches up as the number of goals increases, depending on $\ell$.
Intuitively, this is because BrFS must examine all states in $\shallowstates$ --- shown as the dashed line --- regardless of the goal density.
\crrwspace has no such requirement, as its runtime converges to $\goaldepth$ (\textit{i.e.} the task is solved with the first random walk) as $\numgoals$ increases to $|\goaldepthstates|$.
In this section, we formalize a more general version of this relationship.

We begin by finding a lower bound on the $\successprob$ such that \crrwspace is faster in expectation than BrFS.
Notably, while the results below require some given number of goals at depth $\goaldepth$, they also apply if there are deeper goals.

\begin{theorem}\label{thm:general_success_prob}
Let $\task$ be a search task with $\numgoals \geq 1$ goal states uniformly distributed among the $|\goaldepthstates|$ states at goal depth $\goaldepth \geq 1$. Then $\expect[\rrwrv(\task)] \leq \expect[\bfsrv(\task)]$ if the success probability $\successprob$ of any single random walk satisfies
\begin{align*}
    \successprob \geq \frac{\walklength}{|\shallowstates| + (|\goaldepthstates| + 1)/(\numgoals + 1) - 1}
\end{align*}
\end{theorem}
\begin{proof}
The proof begins by taking the inverse of the inequality on $s$ above, and then performing the following derivation
\begin{align}
    1 / \successprob &\leq (|\shallowstates| + (|\goaldepthstates| + 1)/(\numgoals + 1) - 1)/\walklength\\
    \walklength / \successprob + 1 &\leq |\shallowstates| + (|\goaldepthstates| + 1)/(\numgoals + 1) \label{line:shift_around}\\
    \expect[\rrwrv(\task)] &\leq \expect[\bfsrv(\task_s)] \label{line:replace_exps}
\end{align}
Line \ref{line:shift_around} simply multiples both sides by $\walklength$ and moves the $1$ from one side to the other.
Line \ref{line:replace_exps} then follows immediately from Theorem \ref{thm:brfs_exp} and Theorem \ref{thm:rrw_exp}.
\end{proof}

Now suppose the state-space has the structure of a directed tree and the random walks are unbiased.
Notably, we doe not assume a constant branching factor, and some states in the tree may not even have any successors.
To handle this, let $\goaldepthprob \geq 0$ be the probability of any single random walk reaching the goal depth.
Here, $\successprob \leq \goaldepthprob$ and $\successprob > 0$ since the walks are unbiased.
In this scenario, we now provide a bound on the number of goals at the goal depth which guarantees that \crrwspace is at least as fast in expectation as BrFS.
\begin{theorem}\label{thm:simple_crossover_bound}
Let $\task$ be a search task on a tree with goal depth $\goaldepth \geq 2$ and $1 \leq \numgoals < |\goaldepthstates|$ goals at depth $\goaldepth$.
Let $\goaldepthprob$ be the probability of an unbiased random walk of length $\walklength$ reaching the goal depth, where $0 < \goaldepthprob \leq 1$.
Then $\expect[\bfsrv(\task_\bfactor)] \geq \expect[\rrwrv(\task_\bfactor)]$ if the number of goals $\numgoals$ at the goal depth satisfies 
$\numgoals \geq \walklength |\goaldepthstates|/\left[ \goaldepthprob|\shallowstates| \right]$
\end{theorem}
\begin{proof}
In this case, $\successprob \geq \goaldepthprob \numgoals/ |\goaldepthstates|$, because the probability of reaching the goal depth is $\goaldepthprob$, and the probability of visiting a goal state given that we have reached depth $\goaldepth$ is $\numgoals / |\goaldepthstates|$. This is formally proven in Lemma \ref{lemma:succ_prob} in Appendix \ref{sec:appendix_proofs} for a formal proof of this claim.
Thus, if the number of goals satisfies $\numgoals \geq \walklength |\goaldepthstates|/\left[ \goaldepthprob|\shallowstates| \right]$, then
\begin{align}
    \successprob & \geq \goaldepthprob \numgoals / |\goaldepthstates| \geq \goaldepthprob \frac{\walklength |\goaldepthstates|}{\goaldepthprob|\shallowstates|} / |\goaldepthstates| \geq \walklength / |\shallowstates|\\
    & \geq \walklength / \left[|\shallowstates| + (|\goaldepthstates| + 1)/(\numgoals + 1) - 1 \right] \label{line:add_rest}
\end{align}
The last line holds because $(|\goaldepthstates| + 1)/(\numgoals + 1) - 1 \geq 1$ since $\numgoals < |\goaldepthstates|$.
The result thus follows by Theorem \ref{thm:general_success_prob}.
\end{proof}

We note that by using $1$ as a (possibly loose) lower bound for $(|\goaldepthstates| + 1)/(\numgoals + 1)$ on line \ref{line:add_rest}, the bound on $g$ given in the theorem above can overestimate the actual needed number of goals.
A more accurate bound is given by Theorem \ref{thm:accurate_bound} in Appendix \ref{sec:appendix_proofs}, which uses a minor correction term to better account for the work that BrFS does at the goal depth.
However, we focus on the simpler bound in Theorem \ref{thm:simple_crossover_bound} since the improvement is marginal and the simpler bound already captures the main properties affecting the relative performance of BrFS and RRWs.

Two other notable cases are formalized in Appendix \ref{sec:appendix_proofs}.
First, we show that \crrwspace is usually faster than BrFS if all states in $\goaldepthstates$ are goals (Corollary \ref{cor:all_goals_at_goal_depth}). 
This holds because the first random walk to reach the goal depth will succeed, while BrFS must still examine all states in $\shallowstates$.
Next, we show that if the goal depth is $1$, then BrFS is usually faster than \crrwspace using unbiased random walks (Corollary \ref{lemma:goal_depth_1}).
Intuitively, this is because BrFS will examine the states in $\goaldepthstates$ one-by-one in turn, while \crrwspace will sample these states with replacement through the random walks.

\subsection{Understanding the Bounds} \label{sec:bound_discussion}

Let us now consider several implications of the above analysis.
We first note that intuitively, Theorem \ref{thm:simple_crossover_bound} indicates that the number of goals needed for \crrwspace to outperform BrFS on a directed tree --- which we call the \emph{goal crossover point} --- largely depends on the ratio of the number states at the goal depth ($|\goaldepthstates|$) to the number of states shallower than that ($|\shallowstates|$).
Implicitly, it also depends on $\goaldepth$ since $\ell \geq \goaldepth$.
For example, again consider a task on a directed tree with constant branching factor $\bfactor$ and uniformly distributed goals.
Since $|\shallowstates| = (\bfactor^d - 1)/(\bfactor - 1)$, $|\goaldepthstates| = \bfactor^d$, and $\goaldepthprob = 1.0$ (because all states have at least one successor), Theorem \ref{thm:simple_crossover_bound} states that the goal crossover point is $\ell (\bfactor - 1)\bfactor^d/(\bfactor^d - 1)$.
This is seen in Figure \ref{fig:cutoff_behaviour}, which shows this goal crossover point when $\bfactor = 4$ 
as a function of goal depth and $``\ell$ error".
That is, we assume $\ell = (1.0 + e)\goaldepth$, meaning each line corresponds to overestimating $\goaldepth$ by the same constant factor. 

While the goal crossover point increases linearly with the goal depth, the density $\numgoals/|\goaldepthstates|$ of goals at the goal depth needed for \crrwspace to outperform BrFS actually \emph{decreases} with $\goaldepth$.
This is seen in Figure \ref{fig:density_behaviour}, which shows the \emph{goal density crossover} (\textit{ie.} the goal crossover divided by $|\goaldepthstates|$).

While Theorem \ref{thm:simple_crossover_bound} captures the importance of the relationship between goal density and the ratio of $|\goaldepthstates|/|\shallowstates|$, practical performance when escaping UHRs may differ for several reasons.
For one, \crrwspace can benefit from goals (\textit{ie.} escape states) at depths between $\goaldepth$ and $\ell$ as these will increase the success probability.
BrFS does not benefit from such goals in any way.
On the other hand, BrFS will better handle \emph{transpositions} because it performs duplicate detection.
When there are many transpositions, \crrwspace is effectively searching on a larger search tree than BrFS, and \crrwspace will struggle in a similar manner as IDA$^*$ \cite{korf:ida_star} does on such problems.
Along with the above observations about the goal density crossover, we would therefore expect \crrwspace to more quickly escape UHRs in large combinatorial state-spaces, and BrFS to better handle cases with very few escape states or many transpositions.

\crrwspace may also have a further advantage in terms of wall-clock time since it does not have the additional overhead of maintaining open and closed lists as needed for duplicate checking and other operations.
In PDDL planning, this overhead is likely minimal since runtime dominated by heuristic calculation.
However, these open and closed lists do mean the worst-case memory requirements of BrFS is $O(B^D)$, in contrast to \crrwspace which is only $O(\ell)$ since only a single random walk is in memory at any one time. 
This can make \crrwspace especially useful in low-memory scenarios.

\section{Enforced Hill-Climbing with RRWs} \label{sec:rrw_ehc}
Recall that EHC breaks the search into a sequence of UHRs, where BrFS is used to find an escape state for each.
Thus, we can study the relative effectiveness of BrFS and RRWs for escaping UHRs by comparing standard EHC to variants that use RRWs instead of BrFS to escape the UHRs.
To that end, we introduce EHC-\crrwspace and EHC-\lubyrrw, which only differ from EHC in that they call \crrwspace and \lubyrrw, respectively, instead of BrFS on line \ref{ehc:brfs_call} of Algorithm \ref{alg:ehc}. 
In this section, we identify several formal properties for these variants, and evaluate them on PDDL problems.

\begin{table*}[t!]
    \centering
    \small
    \begin{tabular}{l|l|r|rrr|rrr}
    \toprule
    & & & \multicolumn{3}{c|}{EHC-\crrw} & \multicolumn{3}{c}{EHC-\lubyrrw} \\
    Autoscale Suite&Domain Classification& EHC & $\ell=10$ & $\ell=25$ & $\ell=50$ & $m = 1$ & $m = 2$ & $m = 4$\\ 
    \midrule
    \multirow{3}{*}{Optimal Track}&UHR-Bounded Total (180) & 180.0 & 180.0 & 180.0 & 180.0& 180.0& 180.0& 180.0\\
    &UHR-Unbounded Total (240) & 219.6 & 196.0 & 218.4 & 223.0 & 222.6 & 223.0 & 223.8 \\
    & Incomplete Total (90) & 41.6 & 24.2 & 25.0 & 25.0 & 35.8 & 34.6 & 28.8 \\
    \midrule
    \midrule
    \multirow{3}{*}{Agile Track}&UHR-Bounded Total (153) & 103.6 & 99.6 & 99.6 & 98.0 & 97.2 & 97.4 & 96.4 \\
    &UHR-Unbounded Total (220)& 97.2 & 80.2 & 100.8 & 101.4 & 107.8 & 107.8 & 107.0 \\
    &Incomplete Total (90) & 25.0 & 17.8 & 15.6 & 14.2 & 21.6 & 19.8 & 19.8 \\
    \bottomrule
    \end{tabular}
    \caption{A summary of the coverage of EHC and the EHC-RRW variants on different types of problems.}
    \label{tab:coverage_summary}
\end{table*}

\subsection{Properties of EHC and EHC-RRW Variants} \label{ehc_props}

The EHC-RRW variants have a strong connection to Arvand \cite{nakhost:arvand, nakhost:arvand_2nd_gen} which also performs an RRW-based local search.
However, instead of committing to the first escape state found, Arvand commits to the state with the lowest heuristic value found after a fixed number of walks.
Arvand also only calculates the heuristic value of the last state along every random walk, restarts the entire search back to $\initstate$ when progress stalls, and incorporates a number of other planning enhancements.
While this makes Arvand a powerful planner, these features make it unsuitable for isolating and studying the effectiveness of using RRWs to escape random walks.
Thus, this work focuses on the simpler methods EHC-RRW variants outlined above. 

Our analysis is based on that of
\citeauthor{hoffmann:ignoring_delete_lists} \shortcite{hoffmann:ignoring_delete_lists} who outlined a domain taxonomy --- originally for the \emph{delete relaxation} $h^+$ heuristic --- that categorizes domains according to their topological characteristics.
The first axis of the taxonomy relates to dead-ends. A domain either has no dead-ends, \emph{recognized} dead-ends (\textit{i.e.} $s$ is a dead-end if and only if $h(s) =\infty$), or \emph{unrecognized} dead-ends ($\exists s \in S$ such that $s$ is a dead-end and $h(s)<\infty$). \citeauthor{hoffmann:ignoring_delete_lists} showed that EHC is complete on a state-space topology $\tuple{\task, h}$ if $\task$ has no dead-end states, or all dead-end states are recognized by $h$.
This is because a UHR in any such $\task$ must have a finite exit distance, where recognized dead-ends are treated as states with no successors. As such, there is some maximum exit distance $D(\task)$ over all UHRs in $\task$. 
If $B$ is the maximum number of successors of any state in $\task$, then $|\shallowstates| \in O(B^{D(\task)})$ and $|\goaldepthstates| \in O(B^{D(\task)+1})$, where the extra ``plus one" is because the shallowest escape is one step deeper than the shallowest exit.
Thus, the runtime for each BrFS to escape a UHR will be $O(B^{D(\task)+1})$ by Equation \ref{eqn:brfs_upper_and_lower}.
Since $h$ only returns non-negative integer values, there are at most $h(\initstate)$ UHRs, and so the runtime of EHC is $O(h(\initstate)B^{D(\task) + 1})$.
Thus, EHC is complete on such state-space topologies, which we call \emph{EHC-complete}. Domains with \emph{unrecognized} dead ends will thus be \emph{EHC-incomplete}.

Recall that an algorithm $A$ is complete on a problem if there exists some constant $k\geq 0$, such that $A$ always terminates on that problem in at most $k$ steps.
As the EHC-RRW variants are stochastic, no such constant exists for these methods.
That said, the EHC-RRW variants using unbiased random walks can be shown to have a finite expected runtime on EHC-complete problems.
To see why, consider using EHC-\crrwspace where $\ell \geq D(\task)$.
In this case, $\successprob \geq 1/B^{D(\task)+1}$ since at least one path with depth at most $D (\task)+ 1$ will reach an escape.
Thus, the expected runtime to escape any UHR will be $O(\ell B^{D(\task)+1})$ by Theorem \ref{thm:rrw_exp}, and the expected runtime to solve the problem is $O(h(\initstate)\ell B^{D(\task)+1})$ by the same argument as EHC.

Since the expected runtime of using \lubyrrwspace to escape an UHR will be no worse than a log-factor more the optimal restart policy on that UHR \cite{luby:restarts}, the expected runtime of EHC-\lubyrrwspace runtime is at most a log-factor worse than EHC-\crrwspace with $\ell=D$, by a similar argument as above. As such,
EHC-\crrwspace has a finite expected runtime if $\ell \geq D$ and EHC-\lubyrrwspace will always have a finite expected runtime on EHC-complete problems.

The second axis of \citeauthor{hoffmann:ignoring_delete_lists}'s taxonomy further divides the EHC-complete domains by UHR size. In \emph{bounded-UHR} domains, there exists an integer $D>0$ such that the exit distance of every UHR in every problem in the domain is at most $D$. In contrast, the exit distance of UHRs in \emph{unbounded-UHR} domains can grow arbitrarily with problem size even for problems within the same domain.

To see the value of bounded-UHRs, consider solving STRIPS planning problems --- where the set of operators $\mathcal{O}$ is given as input --- when using the $h^+$ heuristic.
Here, $|\mathcal{O}| = B$, and since no operator in $\mathcal{O}$ can be included more than once in the optimal delete relaxed plan to a problem, $h^+(\initstate) \leq B$.
Thus, because $D$ is independent of the problem input, EHC has a polynomial runtime of $O(B \cdot B^{D+1})$ on problems from bounded-UHR domains.

By a similar analysis as above, EHC-\crrwspace will clearly have a polynomial expected runtime of $O(\ell B^{D + 1})$ on such STRIPS planning problems if $\ell \geq D$, and EHC-\lubyrrwspace will always have a polynomial expected runtime in this case.

\subsection{Empirical Evaluation}

In this section, we evaluate EHC and EHC-RRW on PDDL planning problems to better understand the relative effectiveness of BrFS and RRWs for escaping UHRs.

\subsubsection{Experimental Setup.}
We tested all methods using Fast Downward \cite{helmert-jair2006}.
EHC was re-implemented since the existing version deviated from the original description in several important ways. 
Early experiments suggested our version outperforms the existing one.
The details of our implementation can be found in Appendix \ref{sec:appendix_ehc_implementation}.

The problems used for testing came from the optimal and agile Autoscale benchmark suites \cite{torralba:autoscale}. In particular, we experiment with the 17 domains that have been previously categorized for $h^+$ according to the taxonomy described in Section \ref{ehc_props} \cite{hoffmann:ignoring_delete_lists}.
The categorization of these 17 domains according to the $h^+$ taxonomy can be found in Table \ref{tab:taxonomy} in Appendix \ref{sec:appendix_results}.

All algorithms were tested using the unit-cost FF heuristic \cite{hoffmannff_2001}. While FF only approximates $h^+$, it has previously been shown empirically that the same taxonomy holds for FF on the 17 domains in question \cite{hoffmann:empirical_analysis}.
In addition, if $h^+$ recognizes dead-ends in a domain, then provably so too will the FF heuristic, meaning completeness is not impacted by using FF \cite{hoffmann:ignoring_delete_lists}.

Finally, all experiments were run on a VMware Virtual Platform using an 8-core Intel Xeon Gold 6226R CPU @ 2.90GHz with a 16 KB L1 cache, with a 10 minute time limit and a 3.5
GB memory limit per problem.
Results were averaged over 5 runs per problem, including for EHC which was implemented to use random tie-breaking.
EHC-\crrwspace and EHC-\lubyrrwspace were each tested with three different values for the walk length $\ell$ and the base multiplier $m$, respectively.

\subsubsection{Coverage Results.} Table \ref{tab:coverage_summary} summarizes the coverage results of the different methods on the different test suites used. 
Full per-domain results can be seen in Table \ref{tab:classified_coverage} in Appendix \ref{sec:appendix_results}.
The number of problems per category is shown in parentheses.
The appendix also contains plots that show how coverage changes with number of evaluations and time.
While Autoscale contains 30 problems, Fast Downward was unable to translate some problems in the agile track from PDDL to SAS+ in the given 3.5 GB memory limit.
We omit these from the test set.
Below, we consider each taxonomy category individually to better understand how different topological features impact performance.

\subsubsection{Bounded-UHR Domains.} All optimal track problems in bounded-UHR domains were solved by all tested methods. 
This is consistent with the strong expected runtime guarantees we have for EHC and the EHC-RRW variants.
However, the large size of the agile track problems meant they still remain challenging for EHC-based approaches.
For example, in the \texttt{logistics} domain --- which has a maximum exit distance of $1$ --- has such a high branching factor that only 10s of states were being expanded per second.

\begin{figure}[t]
 	\centering
 	\includegraphics[width=0.725\linewidth]{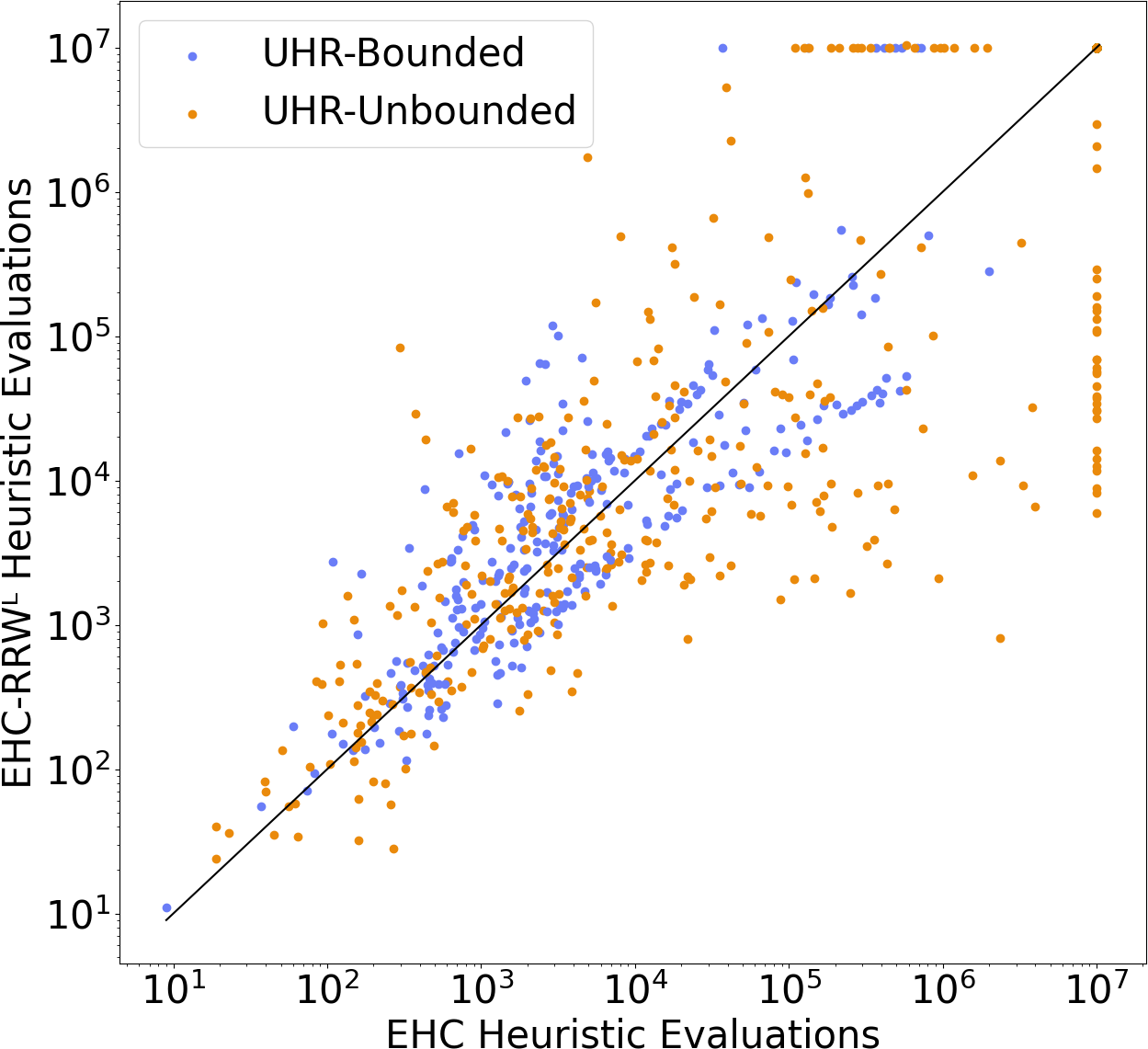} 
    \caption{Per-problem runtime comparison between EHC and EHC-\lubyrrwspace with $m=1$.} \label{fig:ehc_vs_luby}
 \end{figure}

A comparison of per-problem heuristic evaluations (\textit{ie.} runtime) of one run of each of EHC and EHC-\lubyrrwspace with $m=1$ is shown in Figure \ref{fig:ehc_vs_luby}.
The bounded-UHR problems are shown as blue dots.
For these runs, EHC was faster or the only one able to solve $56.3\%$ of the problems.
This advantage is consistent with our theory which suggests a high goal density is needed for \crrwspace to be faster if the escapes are shallow (\textit{ie.} see Figure \ref{fig:density_behaviour}).
The figure also shows there is significant correlation between the performance of these methods, which we verified by calculating the Pearson correlation of the logarithm of these evaluation counts.
The correlation on problems solved by both methods was 0.83.

\subsubsection{Unbounded-UHR Domains.} EHC-\lubyrrwspace has the best coverage in these domains, but this is largely due to the two \texttt{pipesworld} domains.
On the remaining domains, its performance is similar to EHC.
EHC-\crrwspace can show strong coverage, but is very sensitive to the selection of $\ell$.
In terms of runtime, EHC-\lubyrrwspace and EHC are almost exactly equal in how many times each was fastest on all unbounded-UHR problems, but EHC-\lubyrrwspace is better performing on $70.3\%$ of the agile unbounded-UHR problems.
Given our theoretical results, this suggests the goal density is not dropping dramatically as the exit distance increases.

The Pearson correlation between EHC-\lubyrrwspace and EHC on unbounded-UHR problems was $0.68$.
While this is lower than on bounded-UHR problems, it is still reasonably high.

\subsubsection{EHC-Incomplete Domains.} Table \ref{tab:coverage_summary} shows the EHC-based approaches struggle the most on domains with unrecognized dead ends.
EHC has the best performance, albeit on only $3$ domains.
In such domains, it is not only important to escape UHRs as quickly as possible, but to find a ``good" escape that does not lead to a dead end region of the state-space.
We hypothesize that when using $h^+$-based heuristics on such problems, the use of BrFS to find shallowest escapes may have an advantage.
This is because delete relaxation based methods do not recognize when resources (like fuel) are exhausted by an action, but shallower escapes may mean less resources are being used up, and thus EHC is less likely to find an escape leading to a dead end.
However, further investigation is needed on this topic.

\begin{figure}[t]
 	\centering
 	\includegraphics[width=0.8\linewidth]{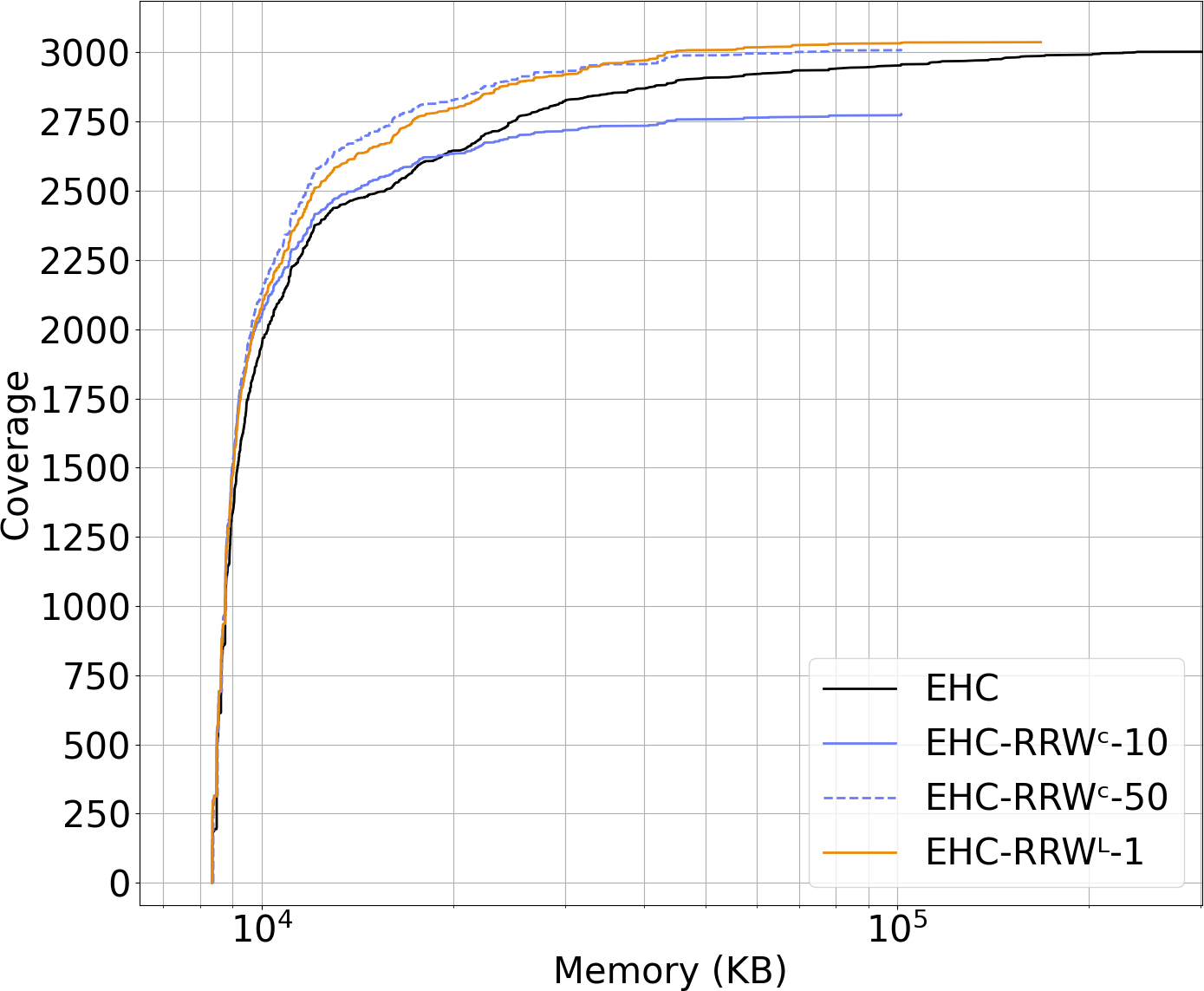} 
    \caption{Memory usage comparison between EHC and the EHC-RRW variants.} \label{fig:memory_usage}
 \end{figure}
 
\subsubsection{Memory Usage.} Figure \ref{fig:memory_usage} shows how coverage relates to memory for the methods discussed on the EHC-complete domains.
For readability, we only include $m=1$ for EHC-\lubyrrwspace since all values had very similar behaviour, and the best and worst performing values of $\ell$ for EHC-\crrw.
The EHC-RRW variants show clear advantages in this evaluation, which aligns with the fact that RRWs do not need to maintain open and closed lists like BrFS.
Notably, none of the methods ran out of memory on any test problem.

\section{Conclusion and Related Work}
In this work, we focus on improving our theoretical and empirical understanding of two different methods for escaping UHRs in BrFS and RRWs.
To do so, we characterized the expected runtime of these approaches, and then showed how the structure of a UHR and the probability that a single random walk will find an escape determine if RRWs will outperform BrFS in expectation.
Next, we considered RRW-based variants of EHC, since this algorithm consists of a series of BrFS, each aiming to escape a UHR.
Existing worst-case runtime results were extended to these variants in the form of expected runtime guarantees.
We also empirically compared EHC and EHC-RRW to better understand their relative behaviour in practice.

Regarding related work, Arvand-LS is similar to Arvand (described in Section \ref{ehc_props}), but it uses local GBFSs augmented with random walks to escape each UHR \cite{xie:local_rws}.
Local best-first searches have also been useful for escaping UHRs in best-first search-based planners \cite{xie:improved_local}.
Our results supplement this research by furthering our understanding of the differences between these competing ways to escape UHRs.

\citeauthor{nakhost:rw_analysis} (\citeyear{nakhost:rw_analysis}) formally analyzed the expected runtime of a single random walk and RRWs on classes of undirected graphs characterized by the probability of getting closer or farther from a goal on every step.
Their model for RRWs assumed a constant restart probability at every step instead of a constant restart depth.
\citeauthor{everitt:bfs_vs_dfs} (\citeyear{everitt:bfs_vs_dfs}) also performed an analysis comparing BrFS and depth-first search on bounded depth-trees.
Their analysis does not include RRWs, and makes different assumptions on the distribution of the goals in the tree.
Understanding the applicability of their goal distribution models to escaping UHRs and RRWs remains as future work.

\bibliography{aaai26}

\newpage
\appendix

\section{Basic Algorithms} \label{sec:appendix_algs}

Pseudocode for BrFS is given below, which defines this algorithm using $\OPEN$ and $\CLOSED$ lists as in \emph{best-first search}.
On every iteration, BrFS will select the state $s$ in $\OPEN$ with the lowest depth (line \ref{alg:brfs:line:min}), and generate its successors (line \ref{alg:brfs:line:generation}).
For any successor that is seen for the first time, the algorithm will perform a goal test. 
If it succeeds, the algorithm terminates and returns the path from $\initstate$ which is implicitly stored using $\parent$ pointers. 
If $s$ is not a goal, then it is added to $\OPEN$, and the search continues.
Notably, we have defined BrFS to perform a goal test on state generation, not after selecting it from $\OPEN$ to minimize the number of generations.
Unlike other best-first search algorithms, BrFS is still guaranteed to find the shortest path to a goal even with this modification.

\begin{algorithm}[h]
\small
\begin{algorithmic}[1]
    \STATE \textbf{Input:} task $\tuple{\states, \initstate, \transitions, \goaltest}$
    \IF{$\goaltest(\initstate) = \true$} \label{brfs:goal_test_start}
        \RETURN $\tuple{\initstate}$ ~~~~~~\% Single state path is solution
    \ENDIF
    \STATE $\parent(\initstate) = \NONE$, $depth(\initstate) = 0$
    \STATE $\OPEN \gets \lbrace \initstate \rbrace$, $\CLOSED \gets \lbrace \rbrace$
    \WHILE{$\OPEN$ is not empty}
        \STATE $s \gets \arg\min_{s' \in \OPEN} depth(s')$ \label{alg:brfs:line:min}
        \STATE $\CLOSED \gets \CLOSED \cup \lbrace s \rbrace$
        \FORALL{$s' \in \transitions(s)$} \label{alg:brfs:line:generation}
            \IF{$s' \notin \OPEN \cup \CLOSED$}
                \STATE $\parent(s') = s$, $depth(s') = depth(s) + 1$
                \STATE $\OPEN \gets \OPEN \cup \lbrace s' \rbrace$
                \IF{$\goaltest(s') = \true$} \label{brfs:goal_test_reg}
                    \RETURN $\tuple{\initstate, \ldots, s'}$ a from $\parent$ pointers
                \ENDIF
            \ENDIF
        \ENDFOR
    \ENDWHILE
    \RETURN $\tuple{}$~~~~~~\% No solution exists
\end{algorithmic}
\caption{Breadth-First Search (\texttt{brFS})}
\label{alg:brfs}
\end{algorithm}

\section{Additional Theoretical Proofs} \label{sec:appendix_proofs}

In this section, we provide additional results to those provided in the main text.
We begin this section with comparative runtime analysis between BrFS and \crrwspace in several notable special cases.
First, we consider the case where all the states at the goal depth of a tree are goals.
This is generally advantageous for \crrw, since the first random walk that reaches that depth will find a goal, while BrFS must examine all states shallower than the goal depth.
\begin{corollary} \label{cor:all_goals_at_goal_depth}
Let $\task$ be a search task on a tree where all states in $\goaldepthstates$ are goals.
If the probability of any random walk reaching a goal is $\goaldepthprob \geq \goaldepth / |\shallowstates|$, then $\expect[\bfsrv(\task_\bfactor)] \geq \expect[\rrwrv(\task_\bfactor)]$.
\end{corollary}
\begin{proof}
Since all states at the goal depth are goals, then a goal will be found on the first random walk to reach the goal depth.
Thus $\goaldepthprob = \successprob$.

Notice that since $\goaldepthprob > 0$ means that $\walklength \geq \goaldepth$.
However, no single random walk can go deeper than $\goaldepth$, so the \crrwspace is essentially running random walks with a maximum depth of $\goaldepth$.
The result then follows immediately from Theorem \ref{thm:general_success_prob}.
\end{proof}

In the special case that every state in $\task$ has at least one successor, this corollary implies that $\expect[\bfsrv(\task_\bfactor)] \geq \expect[\rrwrv(\task_\bfactor)]$ regardless of the size of $\shallowstates$.
This is because $|\shallowstates| \geq d^*$, since there must be at least one state at each of depths $0, 1, \ldots, \goaldepth - 1$.

Next, we consider the case where the goal depth is $1$. This is a notable for EHC-RRW variants, since some problems have been shown to have bounded exit distances.
\begin{corollary}
Let $\task$ be a search task with $\goaldepth = 1$ and $1 \leq \numgoals < |\goaldepthstates|$ at this depth. Then $\expect[\bfsrv(\task_\bfactor)] \geq \expect[\rrwrv(\task_\bfactor)]$ if 
\begin{align}
    s \geq \ell (\numgoals + 1)/(|\goaldepthstates| + 1)
\end{align}
\end{corollary}
This follows directly from Theorem \ref{thm:general_success_prob}, because  $\initstate$ is the only state shallower than the goal depth, meaning $|\shallowstates| = 1$.

Notably, if \crrwspace is using unbiased random walks, not all states at depth $1$ are goal states, and there are no goal states in $\task$ reachable at any depth deeper than depth $1$, then we can guarantee BrFS will always be faster in expectation than \crrw.
Intuitively, this is because both algorithms visit the initial state once, and then BrFS samples states from the goal depth in search of a goal ``without replacement" (\textit{i.e.} it never performs a goal test on any state more than once).
In contrast, \crrwspace searches the goal depth ``with replacement", since different random walks may visit the same state.
This is formalized in the following lemma.

\begin{lemma} \label{lemma:goal_depth_1}
If $\task$ is a search task with $\goaldepth = 1$, $1 \leq \numgoals < |\goaldepthstates|$ goals at this depth, and no goal states reachable on paths containing two or more transitions, then $\expect[\bfsrv(\task_\bfactor)] < \expect[\rrwrv(\task_\bfactor)]$ if the random walks are unbiased.
\end{lemma}
\begin{proof}
We will show the statement holds for $\walklength = 1$.
Since the expected runtime of \crrwspace can only be higher for a given walk length because all goals are at depth 1, this will show the result holds in general.

By Theorem \ref{thm:brfs_exp}, the expected runtime of BrFS will be
\begin{align}
    |\shallowstates| + \frac{|\goaldepthstates| + 1}{\numgoals + 1} = 1 + \frac{|\goaldepthstates| + 1}{\numgoals + 1}
\end{align}
This holds since $\shallowstates$ consists solely of the initial state.
The expected runtime of \crrwspace with unbiased random walks to depth $1$ will exactly $1 +|\goaldepthstates| / \numgoals$ by a similar argument to Theorem \ref{thm:rrw_exp}, since $\ell = 1 = \goaldepth$ and every random walk will have a depth of exactly $1$ (since the initial state has at least one successor).
Now, $|\goaldepthstates| > \numgoals$, we are guaranteed that $|\goaldepthstates|/ \numgoals$ is greater than $(|\goaldepthstates| + 1)/(\numgoals + 1)$.
This is because for any $x > y$ for positive $x$ and $y$, $x/y \geq (x + 1) / (y + 1)$, since the right hand side of this inequality moves closer to $1$ by adding $1$ to the top and bottom of the left hand side.
Thus the statement holds for $\walklength = 1$, which as we argued above says the statement holds for any $\walklength$.
\end{proof}

et us now turn to our main result which compares the expected runtime of BrFS and \crrwspace using unbiased random walks on trees.
This is a more accurate version of the result that appears in the main text.
To prove this result, we need several lemmas.
The first provides a lower bound on the success probability of an unbiased random walk on a tree in terms of the goal depth size, number of goals at that depth, and probability a random walk reaches that depth.

\begin{lemma}\label{lemma:succ_prob}
Let $\task$ be a search task on a tree with $\numgoals \geq 1$ goals uniformly distributed among the $|\goaldepthstates|$ states at the goal depth $\goaldepth$. Suppose $0 < \goaldepthprob \leq 1$ is the probability of a single unbiased random walk reaching the goal depth.
Then the success probability of a unbiased random walk $\successprob$ satisfies the following:
\begin{align*}
    \successprob \geq \frac{\goaldepthprob \numgoals}{|\goaldepthstates|}
\end{align*}
\end{lemma}
\begin{proof}
We first show the statement is true in the case that there are no goals deeper than the goal depth, and then extend this result to the general case below.
As such, we first assume there are no goals deeper than the goal depth.
In this case, the success probability of a random walk corresponds to the intersection of two events.
First, the random walk must reach the goal depth.
We will let $\prob[|P| \geq \goaldepth]$ to denote the probability of this happening, which is $\goaldepthprob$ by definition.
Second, the state at the goal depth that has been reached must be a goal state. We denote this as $\prob[\goaltest(last(P)) = \text{True} \mid |P| \geq \goaldepth]$.
Given that the goal depth is reached and each such path only visits one state at the goal depth, this value is $\numgoals/|\goaldepthstates|$ since the goals are uniformly distributed.
\begin{align}
    \successprob & = \prob[|P| \geq \goaldepth] \prob[\goaltest(last(P)) = \text{True} \mid |P| \geq \goaldepth]\\
    & =  \goaldepthprob \numgoals / |\goaldepthstates|
\end{align}
Thus the statement holds when there are no goal states deeper than the goal depth.

If there are goals deeper than $\goaldepth$, the $\prob[\goaltest(P)]$ can only increase and thus the statement holds more generally.
\end{proof}

Next, we show that, all else being equal, while the expected runtime of both BrFS and \crrwspace decrease as goals are added at the goal depth, these additional goal states ``help" \crrwspace more than BrFS.
That is, the expected runtime of \crrwspace decreases faster than the expected runtime for BrFS with each additional goal state at the goal depth assuming everything else about the task stays the same.
We do so by showing that the expression for $\expect[\bfsrv(\task)] - \expect[\rrwrv(\task)]$ increases as the number of goal states (\textit{i.e.} \numgoals) at the goal depth increases.
Since both $\expect[\rrwrv(\task)$] and $\expect[\bfsrv(\task)]$ decrease with an increasing number of goals, the difference between these expectations can only increase if the expected runtime of \crrwspace is decreasing faster with $\numgoals$ than BrFS.
This is done by showing that the derivative of this expression with respect to $\numgoals$ is positive.
Note that in the below expression, $N$ takes on the role of $|\shallowstates|$, $D$ takes on the role of $|\goaldepthstates|$, and $L$ takes on the role of $\ell$.
These are renamed since they are constants in the provided general expression.

\begin{lemma} \label{lemma:dec_with_goals}
Consider the following expression where $N \geq 1$, $D \geq 1$, and $L \geq 2$:
\begin{align}
f(\numgoals) =  N + \frac{D + 1}{\numgoals + 1} - \frac{L D}{\numgoals} - 1
\end{align}
Then $\frac{df(\numgoals)}{d\numgoals} > 0$ for all $1 \leq g \leq D$.
\end{lemma}
\begin{proof}
The derivation is as follows:
\begin{align}
    \frac{df(\numgoals)}{d\numgoals} & = - \frac{D + 1}{(\numgoals + 1)^2} + \frac{LD}{\numgoals^2} \\
    & \geq - \frac{D + 1}{(\numgoals + 1)^2} + \frac{D}{\numgoals^2}  \label{eqn:ed_1}\\
    & = \frac{- g^2 D - g^2 + g^2 D + 2 g D+ D}{g^2(g+1)^2} \\
    & = \frac{2gD - g^2 + D}{g^2(g+1)^2} \geq \frac{2g^2 - g^2 + g}{g^2(g+1)^2} \label{eqn:bound_on_goaldepthsize}\\
    & \geq \frac{g^2 + g}{g^2(g+1)^2} > 0
\end{align}
Line \ref{eqn:ed_1} holds since $L \geq 1$, while line \ref{eqn:bound_on_goaldepthsize} follows since $\numgoals \leq D$.
The final line holds since $\numgoals \geq 1$.
\end{proof}

We can now prove our main result.
\begin{theorem} \label{thm:accurate_bound}
Let $\task$ be a search task on a tree with goal depth $\goaldepth \geq 2$ and $1 \leq \numgoals < |\goaldepthstates|$ goals at this depth.
Let $\goaldepthprob$ by the probability of an unbiased random walk of length $\walklength$ reaching the goal depth, where $0 < \goaldepthprob \leq 1$.
If the following condition on the number of goals is satisfied
\begin{align*}
\numgoals \geq \frac{\walklength |\goaldepthstates|}{\goaldepthprob(|\shallowstates| + \kappa - 1)}
\end{align*}
where $\kappa = \max(1, \frac{|\goaldepthstates| + 1}{\walklength|\goaldepthstates|/(\goaldepthprob|\shallowstates|) + 1})$, then $\expect[\bfsrv(\task_\bfactor)] \geq \expect[\rrwrv(\task_\bfactor)]$.
\end{theorem}
\begin{proof}
We will first show the statement holds in the case that $\numgoals$ at the goal depth is exactly equal to the expression in the theorem statement.
Then we will use Lemma \ref{lemma:dec_with_goals} to show it holds for more goals.

Assume that $\numgoals = (\walklength |\goaldepthstates|)/(\goaldepthprob(|\shallowstates| + \kappa - 1))$.
By Lemma \ref{lemma:succ_prob}, this means we can perform the following derivation on the success probability: 
\begin{align}
    \successprob & \geq \goaldepthprob \numgoals / |\goaldepthstates|\\
    & \geq \left[ \frac{\goaldepthprob \walklength |\goaldepthstates|}{\goaldepthprob(|\shallowstates| + \kappa - 1)} \right] / |\goaldepthstates|\\
    & \geq \frac{\ell}{|\shallowstates| + \kappa - 1} \label{line:eqn_with_kappa}
\end{align}
We will now show that $\kappa \leq (|\goaldepthstates| + 1)/(\numgoals + 1)$ so that we can use Theorem \ref{thm:general_success_prob} to get the desired result.
To do so, we first note that due to the $\max$, $\kappa \geq 1$.
As such, $\kappa - 1 \geq 0$, meaning the following derivation is possible:
\begin{align}
    \numgoals & = \frac{\walklength |\goaldepthstates|}{\goaldepthprob(|\shallowstates| + \kappa - 1)}\\
    & \leq \frac{\walklength |\goaldepthstates|}{\goaldepthprob|\shallowstates|}\\
    \numgoals + 1 & \leq \frac{\walklength |\goaldepthstates|}{\goaldepthprob|\shallowstates|} + 1\\
    \frac{|\goaldepthstates| + 1}{\numgoals + 1} & \geq \frac{|\goaldepthstates| + 1}{\walklength |\goaldepthstates|/(\goaldepthprob|\shallowstates|) + 1}\\
    \frac{|\goaldepthstates| + 1}{\numgoals + 1} & \geq \kappa
\end{align}
Along with line \ref{line:eqn_with_kappa}, this implies the following:
\begin{align}
    \successprob & \geq \frac{\ell}{|\shallowstates| + (|\goaldepthstates| + 1)/(\numgoals + 1) - 1}
\end{align}
By Theorem \ref{thm:general_success_prob}, this means that $\expect[\bfsrv(\task_\bfactor)] \geq \expect[\rrwrv(\task_\bfactor)]$ if $\numgoals = (\walklength |\goaldepthstates|)/(\goaldepthprob(|\shallowstates| + \kappa - 1))$.
Now by setting $N = |\shallowstates|$, $D = |\goaldepthstates|$, and $L = \walklength$, Lemma \ref{lemma:dec_with_goals} indicates that $\expect[\bfsrv(\task_\bfactor)] - \expect[\rrwrv(\task_\bfactor)]$ increases as $\numgoals$ increases. As such, if $\expect[\bfsrv(\task_\bfactor)] \geq \expect[\rrwrv(\task_\bfactor)]$ for $\numgoals = (\walklength |\goaldepthstates|)/(\goaldepthprob(|\shallowstates| + \kappa - 1))$, it is necessarily the case that $\expect[\bfsrv(\task_\bfactor)] \geq \expect[\rrwrv(\task_\bfactor)]$ for $\numgoals \geq (\walklength |\goaldepthstates|)/(\goaldepthprob(|\shallowstates| + \kappa - 1))$.
Thus the statement holds.
\end{proof}

\section{Implementation Details} \label{sec:appendix_ehc_implementation}
We implemented all algorithms in the FastDownward Planning Systems \cite{helmert-jair2006}. 
This included a re-implementation of standard EHC. 
The most notable change we made was removing the global node table which facilitated the closed list in the original implementation. 
As a result, if a state was found while exploring one UHR, it would be viewed as closed (and not re-expanded), if it was encountered later in another UHR.
The result is that the he closed list was sometimes ``blocking" the local BrFS. 
In some cases this lead to domains with no dead-ends, such as blocksworld, being unsolvable by EHC if the only solution path happened to pass through the global closed list. 
Thus, EHC was no longer complete on problems it should be.
In our implementation, each local search maintains a local open and closed list that are destroyed upon the next improving state being found. 
In some preliminary testing, our fresh implementation had consistently improved coverage compared to the original implementation.

The original implementation of EHC performed delayed heuristic evaluation and optionally used preferred operators.
As our study sought to isolate for the local search strategy, we removed these enhancements. 
We also implement EHC to do random tie-breaking.
This is done with two queues.
States in the current queue are randomly selected and expanded and their successors are placed in the next queue.
When the current queue is empty, the queues are swapped. 
Our implementation also performs goal testing when a state is generated. 

Code for our implementation has been included.
While we do not include the entire Fast Downward code base, we include all the files relevant for EHC and our EHC-RRWs.

\section{Additional Experiment Results} \label{sec:appendix_results}
We tested our algorithms using the Autoscale Benchmark 21.11 suite of domains \cite{torralba:autoscale}.
Both the optimal and agile tracks were used. Summarizes of the experiments were presented in the main text.
Here, we begin by showing the $h^+$ categorization of the 17 domains we focused our testing on.
These are shown in Table \ref{tab:taxonomy}.

Next, we show the full per-domain results.
These are shown in Table \ref{tab:classified_coverage}.
The number of problems attempted is indicated in parentheses.
While all domains contained 30 problems, some ran out of memory during translation and so fewer than 30 were attempted.

\begin{table}[htbp]
\small
\begin{tabular}{@{}l|l:l}
 & \multicolumn{1}{c|}{Bounded UHRs} & \multicolumn{1}{c}{Unbounded UHRs} \\ \hline
EHC-Complete & \begin{tabular}[c]{@{}l@{}}elevators\\ gripper\\ logistics \\ miconic\\ satellite \\zenotravel\end{tabular} & \begin{tabular}[c]{@{}l@{}}blocksworld\\ depots\\ driverlog\\grid\\pipesworld-Tank\\pipesworld-NoTank\\rovers\\transport\end{tabular} \\ \cline{1-1} \cdashline{2-3}
EHC-Incomplete &  & \begin{tabular}[c]{@{}l@{}}airport\\freecell\\mprime\end{tabular}

\end{tabular}
\caption{The taxonomy of the 17 Autoscale domains with a known classification for $h^+$.}\label{tab:taxonomy}
\end{table}

\begin{table*}[t!]
    \centering
    \small
    \begin{tabular}{l|r|rrr|rrr}
    \toprule
    & & \multicolumn{3}{c|}{EHC-\crrw} & \multicolumn{3}{c}{EHC-\lubyrrw} \\
    Domain & EHC & $\ell=10$ & $\ell=25$ & $\ell=50$ & $m = 1$ & $m = 2$ & $m = 4$\\ 
    \midrule
    ~~~~~~\textit{Optimal Track Problems} & & & & & & & \\
    elevators (30) & 30.0 & 30.0 & 30.0 & 30.0 & 30.0 & 30.0 & 30.0 \\
    gripper (30) & 30.0 & 30.0 & 30.0 & 30.0 & 30.0 & 30.0 & 30.0 \\
    logistics (30) & 30.0 & 30.0 & 30.0 & 30.0 & 30.0 & 30.0 & 30.0 \\
    miconic (30) & 30.0 & 30.0 & 30.0 & 30.0 & 30.0 & 30.0 & 30.0 \\
    satellite (30) & 30.0 & 30.0 & 30.0 & 30.0 & 30.0 & 30.0 & 30.0 \\
    zenotravel (30) & 30.0 & 30.0 & 30.0 & 30.0 & 30.0 & 30.0 & 30.0 \\
    \textbf{UHR-Bounded Total (180)} & 180.0 & 180.0 & 180.0 & 180.0& 180.0& 180.0& 180.0\\
    & & & & & & & \\
    blocksworld (30) & 24.4 & 18.4 & 25.8 & 28.4 & 26.6 & 26.6 & 26 \\
    depots (23) & 25.8 & 16.6 & 20.0 & 20.4 & 21.6 & 22.0 & 23.2 \\
    driverlog (30) & 30.0 & 30.0 & 30.0 & 30.0& 29.8 & 29.8 & 29.8 \\
    grid (25) & 29.8 & 28.2 & 29.8 & 30.0 & 29.8 & 29.8 & 29.8 \\
    pipesworld-notankage (30) & 29.8 & 25.4 & 29.8 & 30.0 & 30.0 & 30.0 & 30.0 \\
    pipesworld-tankage (30) & 19.8 & 17.8 & 23.4 & 24.2 & 25.0 & 25.0 & 25.2 \\
    rovers (30) & 30.0 & 30.0 & 30.0 & 30.0 & 30.0 & 30.0 & 30.0 \\
    transport (22) & 30.0 & 29.6 & 29.6 & 30.0 & 29.8 & 29.8 & 29.8 \\
    \textbf{UHR-Unbounded Total (240)} & 219.6 & 196.0 & 218.4 & 223.0 & 222.6 & 223.0 & 223.8 \\
    & & & & & & & \\
    airport & 8.2 & 3.4 & 2.6 & 2.6 & 7.2 & 6.0 & 5.0 \\
    freecell & 11.8 & 6.4 & 10.2 & 11.4 & 10.4 & 11.8 & 8.4 \\
    mprime & 21.6 & 14.4 & 12.2 & 11.0 & 18.2 & 16.8 & 15.4 \\
    \textbf{Incomplete Total (90)} & 41.6 & 24.2 & 25.0 & 25.0 & 35.8 & 34.6 & 28.8 \\
    \midrule
    \midrule
    ~~~~~~\textit{Agile Track Problems} & & & & & & & \\
    elevators (30) & 22.2 & 20.0 & 20.8 & 19.8 & 18.6 & 18.8 & 18.0 \\
    gripper (30) & 30.0 & 30.0 & 30.0 & 30.0 & 30.0 & 30.0 & 30.0 \\
    logistics (30) & 6.0 & 5.2 & 5.0 & 4.4 & 5.0 & 4.6 & 4.6 \\
    miconic (30) & 30.0 & 30.0 & 30.0 & 30.0 & 30.0 & 30.0 & 30.0 \\
    satellite (18) & 8.0 & 7.4 & 7.0 & 7.0 & 7.2 & 7.0 & 7.2 \\
    zenotravel (15) & 7.4 & 7.0 & 6.8 & 6.8 & 6.4 & 7.0 & 6.6 \\
    \textbf{UHR-Bounded Total (153)} & 103.6 & 99.6 & 99.6 & 98.0 & 97.2 & 97.4 & 96.4 \\
    &&&&&&&\\
    blocksworld (30) & 3.0 & 1.0 & 4.6 & 3.4 & 3.4 & 3.6 & 3.2 \\
    depots (23) & 7.4 & 5.0 & 5.8 & 6.4 & 6.6 & 6.8 & 5.8 \\
    driverlog (30) & 8.8 & 5.2 & 4.8 & 4.4 & 2.2 & 2.8 & 3.6 \\
    grid (25) & 6.4 & 4.6 & 6.0 & 6.4 & 6.0 & 6.0 & 6.2 \\
    pipesworld-notankage (30) & 22.6 & 16.0 & 22.0 & 23.0 & 29.4 & 29.0 & 29.0 \\
    pipesworld-tankage (30) & 13.6 & 15.0 & 23.6 & 23.6 & 25.8 & 25.8 & 25.2 \\
    rovers (30) & 30.0 & 30.0 & 30.0 & 30.0 & 30.0 & 30.0 & 30.0 \\
    transport (22) & 5.4 & 3.4 & 4.0 & 4.2 & 4.4 & 3.8 & 4.0 \\
    \textbf{UHR-Unbounded Total (220)}& 97.2 & 80.2 & 100.8 & 101.4 & 107.8 & 107.8 & 107.0 \\
    &&&&&&&\\
    airport (30) & 8.6 & 3.0 & 2.0 & 3.0 & 7.2 & 5.0 & 5.2 \\
    freecell (30) & 9.6 & 10.4 & 9.2 & 8.8 & 9.6 & 9.8 & 9.8 \\
    mprime (30) & 6.8 & 4.4 & 4.4 & 2.4 & 4.8 & 5.0 & 4.8 \\
    \textbf{Incomplete Total (90)} & 25.0 & 17.8 & 15.6 & 14.2 & 21.6 & 19.8 & 19.8 \\
    \bottomrule
    \end{tabular}
    \caption{Coverage of EHC and the EHC-RRW variants on the 17 domains in the Autoscale benchmark suite classified in the $h^+$ taxonomy. Includes problems from both the optimal and agile track.}
    \label{tab:classified_coverage}
\end{table*}

\end{document}